\documentclass{article}

\usepackage[final]{neurips_2022}

\usepackage{enumitem}
\usepackage[utf8]{inputenc} 
\usepackage[T1]{fontenc}    
\usepackage{hyperref}       
\usepackage{url}            
\usepackage{booktabs}       
\usepackage{amsfonts}       
\usepackage{nicefrac}       
\usepackage{microtype}      
\usepackage[dvipsnames,table,xcdraw]{xcolor}       
\usepackage{amsmath,bm}
\usepackage{amssymb}
\usepackage{graphicx}
\usepackage{float}
\usepackage{amsthm}
\usepackage{caption}
\usepackage{subcaption}
\usepackage{bbm}
\usepackage{pgf}
\usepackage{algorithm}
\usepackage{algpseudocode}
\hypersetup{
    colorlinks,
    linkcolor={red!50!black},
    citecolor={blue!50!black},
    urlcolor={blue!80!black}
}

\usepackage[utf8]{inputenc}\DeclareUnicodeCharacter{2212}{-}

 \newcommand\Algphase[1]{%
\vspace*{-.7\baselineskip}\Statex\hspace*{\dimexpr-\algorithmicindent-2pt\relax}\rule{\textwidth}{0.4pt}%
\Statex\hspace*{-\algorithmicindent}\textit{#1}%
\vspace*{-.7\baselineskip}\Statex\hspace*{\dimexpr-\algorithmicindent-2pt\relax}\rule{\textwidth}{0.4pt}%
}

\usepackage{tikz}
\usetikzlibrary{shapes,decorations,arrows,calc,arrows.meta,fit,positioning}
\tikzset{
    -Latex,auto,node distance =1 cm and 1 cm,semithick,
    state/.style ={ellipse, draw, minimum width = 0.7 cm},
    point/.style = {circle, draw, inner sep=0.04cm,fill,node contents={}},
    bidirected/.style={Latex-Latex,dashed},
    el/.style = {inner sep=2pt, align=left, sloped}
}

\author{%
  Vahid Balazadeh$^{1}$ \quad Vasilis Syrgkanis$^{2}$\quad Rahul G. Krishnan$^{1}$\\
  $^1$University of Toronto, Vector Institute \quad $^2$Stanford University\\
  \texttt{\{vahid, rahulgk\}@cs.toronto.edu} \\
   \texttt{vsyrgk@stanford.edu} \\
}

\newcommand{\xhdr}[1]{\noindent{{\bf #1.}}}
\newcommand{\Gcal}{\mathcal{G}}
\newcommand{\Mcal}{\mathcal{M}}
\newcommand{\Fcal}{\mathcal{F}}

\newcommand{\E}{\mathbb{E}}
\newcommand{\R}{\mathbb{R}}
\newcommand{\vect}[1]{\ensuremath{\mathbf{#1}}}

\definecolor{Awesome}{rgb}{1.0, 0.13, 0.32}

\newcommand{\lowerATD}{\underline{\text{ATD}}}
\newcommand{\upperATD}{\overline{\text{ATD}}}

\newcommand{\lowerATE}{\underline{\text{ATE}}}
\newcommand{\upperATE}{\overline{\text{ATE}}}

\newcommand{\scm}{\mathcal{M}}
\newcommand{\graph}{\Gcal}

\newtheorem{corollary}{Corollary}
\newtheorem{lemma}{Lemma}

\newtheorem{assump}{Assumption}

\newtheorem{definition}{Definition}
\newtheorem{theorem}{Theorem}

\title{Partial Identification of Treatment Effects with Implicit Generative Models}
\begin{document}
\maketitle

\begin{abstract}
We consider the problem of partial identification, the estimation of bounds on the treatment effects from observational data. Although studied using discrete treatment variables or in specific causal graphs (e.g., instrumental variables), partial identification has been recently explored using tools from deep generative modeling. We propose a new method for partial identification of average treatment effects (ATEs) in general causal graphs using implicit generative models comprising continuous and discrete random variables. Since ATE with continuous treatment is generally non-regular, we leverage the partial derivatives of response functions to define a regular approximation of ATE, a quantity we call \emph{uniform average treatment derivative} (UATD). We prove that our algorithm converges to tight bounds on ATE in linear structural causal models (SCMs). For nonlinear SCMs, we empirically show that using UATD leads to tighter and more stable bounds than methods that directly optimize the ATE.~\footnote{Our code is accessible at \url{https://github.com/rgklab/partial_identification}}
\end{abstract}
\section{Introduction} 

Estimating average treatment effects (ATEs) is a common task that arises in fields involving decision-making, such as healthcare and economics. In the presence of the gold-standard randomized controlled trial (RCT) data, one can compare the outcome variable between treated and control groups to make decisions. But RCTs can be costly to set up and run and are, in many circumstances, infeasible. Consequently, communities are using observational data to assist in decision-making.

Identification of treatment effects from observational data is tied to the structure of the causal graph. For example, the treatment $T$ and outcome $Y$ in Figure~\ref{fig:iv-graph} are confounded by an unobserved random variable, making it impossible to find the causal effect of $T$ on $Y$ only from observational data. 
On the other hand, Figure \ref{fig:backdoor-graph} is identifiable, and one can adjust for confounders using the Back-door formula~\citep{causal-pearl}. 
Even in identifiable settings, non-parametric estimations such as Back-door adjustment formula can point-identify the ATE only with additional assumptions such as positivity, i.e., $P(T=t|X) > 0$ for all values of covariate $X$. Observational data is finite, high-dimensional, and consequently can suffer from severe violations of such assumptions~\citep{d2021overlap}.

\begin{figure}[t]%
\begin{subfigure}[b]{0.25\linewidth}
\centering
    \begin{tikzpicture}
    \node[state] (t) {$T$};
    \node[state] (x) [right =of t, xshift=-0.5cm] {$X$};
    \node[state] (y) [right =of x, xshift=-0.5cm] {$Y$};
    \path (t) edge (x);
    \path (x) edge (y);
    \path[bidirected] (t) edge[bend left=60] (y);
    \path[bidirected] (x) edge[bend left=60] (y);
\end{tikzpicture}
\caption{\small Leaky mediation}
\vspace{1em}
\label{fig:leaky-graph}
\end{subfigure}%
\hspace{1em}
\begin{subfigure}[b]{0.25\linewidth}
\centering
    \begin{tikzpicture}
    \node[state] (z) at (0,0) {$X$};
    \node[state] (t) [right =of z, xshift=-0.5cm] {$T$};
    \node[state] (y) [right =of t, xshift=-0.5cm] {$Y$};
    
    \path (t) edge (y);
    \path (z) edge (t);
    \path[bidirected] (t) edge[bend left=60] (y);
\end{tikzpicture}
\caption{\small IV}
\label{fig:iv-graph}
\vspace{1em}
\end{subfigure}%
\begin{subfigure}[b]{0.25\linewidth}
\centering
\begin{tikzpicture}
    \node[state] (x) at (0,0) {$X$};
    \node[state] (y) [below right =of x, xshift=-0.7cm, yshift=0.5cm] {$Y$};
    \node[state] (t) [below left =of x, xshift=0.7cm, yshift=0.5cm] {$T$};

    \path (x) edge (y);
    \path (x) edge (t);
    \path (t) edge (y);
\end{tikzpicture}%
\caption{\small Back-door}
\label{fig:backdoor-graph}
\vspace{1em}
\end{subfigure}%
\begin{subfigure}[b]{0.25\linewidth}
\centering
\begin{tikzpicture}
    \node[state] (t) at (0,0) {$T$};
    \node[state] (w) [right =of t, xshift=-0.5cm] {$X$};
    \node[state] (y) [right =of w, xshift=-0.5cm] {$Y$};
    
    \path (t) edge (w);
    \path (w) edge (y);
    \path[bidirected] (t) edge[bend left=60] (y);
\end{tikzpicture}
\caption{\small Front-door}
\label{fig:frontdoor-graph}
\vspace{1em}
\end{subfigure}
\begin{subfigure}[b]{\linewidth}
    \centering
    \includegraphics[width=0.85\linewidth]{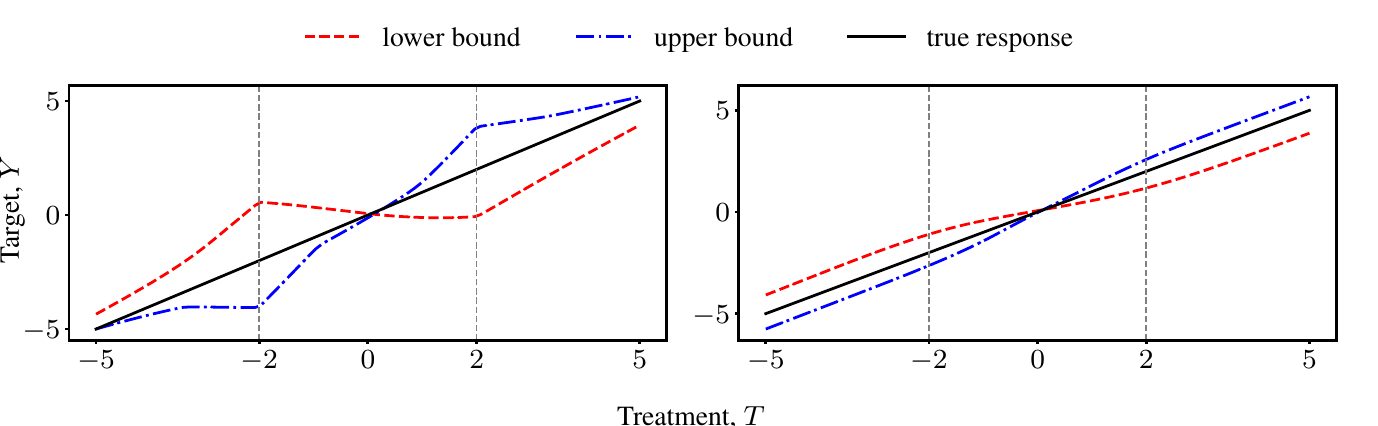}
    \caption{\small Partial identification of ATE in a finite linear Back-door dataset.}
    \label{fig:direct-derivative}
\end{subfigure}
\vspace{0.5em}
\caption{\small The causal graphs for non-identifiable (a) Leaky mediation and (b) Instrumental Variable (IV), and identifiable (c) Back-door and (d) Front-door settings. $T$, $X$, and $Y$ represent treatment, covariates, and target variables. The dashed double-arrows represent latent factors. (e) The response functions corresponding to partial identification of $\E[Y_{T=2}] - \E[Y_{T=-2}]$ in a Back-door linear SCM after training a generative model to match the distribution. (Left) shows the results for directly optimizing the ATE, which leads to a non-informative bound due to the irregularity of ATE with continuous treatments. (Right) is our solution by optimizing the UATD, which results in a tighter bound. Each point $(t, y)$ represents the expected outcome $y$ after intervention $T=t$ in the learned generative model.}
\label{fig:graphs}
\end{figure}
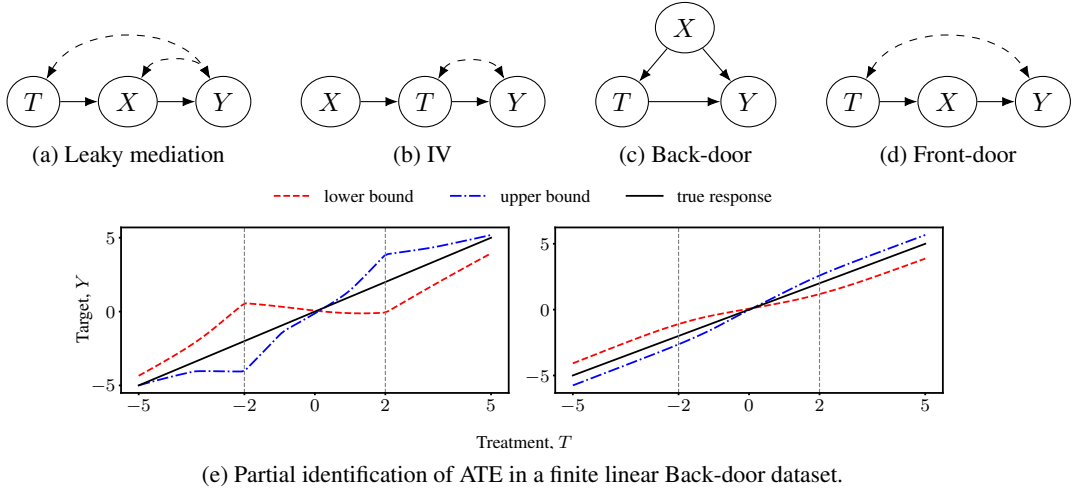

In lieu of the challenges of point-identification, there has been a recognition that decisions can be justified using reliable bounds on the ATE rather than its exact value. For an oncologist treating a cancer patient, knowing that a drug has a significant, positive reduction in the patient's risk of progression may suffice as a rationale to prescribe that drug. This problem is known as partial identification~\citep{manski2003partial}. 
Most existing methods for bounding the ATEs are only applicable in discrete/binary treatment variables~\citep{makermaggie,partial-ident-elias,automated,guo2022partial}. There has been recent interest in continuous treatment settings. However, such methods are applicable for special causal graphs such as the instrumental variables (IV) setting ~\citep{gunsilius-iv, iv-continuous} or make parametric assumptions on the family of treatment-response functions~\citep{padh2022stochastic}. An exception is the work by~\citet{hu2021generative} which provides a non-parametric approach for partial identification using generative adversarial networks (GANs). However, they only provide convergence guarantees for the special case of IV causal graphs.

Using the framework of structural causal models (SCMs) and causal graphs, one can see partial identification as a constrained optimization problem, where the objective, i.e., maximizing/minimizing the ATE, can be written as a post-intervention function of exogenous noise (a.k.a response function) and the constraint is to match the generated samples with the observational distribution. This naturally leads to using generative neural networks such as neural causal models (NCMs)~\citep{causalneural}. 
We find that directly solving the ATE optimization using flexible generative models such as GANs can lead to non-informative and degenerate solutions. The flexibility afforded by generative models such as GANs allows them to deviate significantly from the true response curve in the neighborhood of intervention points to maximize/minimize the ATE while continuing to generate samples akin to the data distribution. Figure \ref{fig:direct-derivative} (left) showcases a typical solution to the ATE optimization.


Our insight is that the ATE between any two points can be approximated as an integral over the derivatives of the response function w.r.t. the treatment variable. Rather than directly optimizing the ATE, we optimize the partial derivatives of the response function, a quantity that we refer to as the uniform average treatment derivative (UATD).~\footnote{Average treatment derivative is also known as average partial effect in the literature~\citep{powell1989semiparametric,wooldridge2005unobserved,rothenhausler2019incremental}.} By optimizing the UATD, the model is required to maximize/minimize the partial derivatives for all points within the treatment support, avoiding extreme local solutions as shown in Figure~\ref{fig:direct-derivative} (right). Our contributions are as follows:

\begin{itemize}[leftmargin=*]
    \setlength\itemsep{0pt}
    \item We formally define the partial identification of average treatment effects as a distributionally-constrained optimization problem, where we choose Wasserstein distance as our constraint metric.
    \item For the class of linear SCMs, we prove that the solution to our optimization problem converges to optimal bounds on the true value of ATE in infinite data for general causal graphs.
    \item We use the solution to partial identification of UATDs to find informative bounds on the value of ATE. We introduce a practical algorithm to solve the distributionally-constrained optimization problem using the Lagrange multiplier formulation with alternating optimization. We empirically show that our algorithm results in tighter and more stable bounds than methods that directly optimize the ATE. 
\end{itemize}

\section{Problem Setup \& Background}
We introduce the definitions and assumptions we will use throughout the paper. Consider the observed data as (possibly continuous) random variables $\vect{V} = \{X_1, \cdots, X_m, T, Y\} \in \R^d$, where $T$, $Y$, and $\{X_1, \cdots, X_m\}$ denote the treatment variable, target variable, and covariates, respectively.

\xhdr{Data generating model}
Our approach will be based on the framework of Structural Causal Models (SCMs). An SCM is a tuple $\scm = (\vect{V}, \vect{U}, \Fcal, P_{\vect{U}})$, where each observed variable $V_i \in \vect{V}$ is a deterministic function of a subset of  variables $\vect{pa}(V_i) \subseteq \vect{V}$ and latent variables $\vect{U}_{V_i} \subseteq \vect{U}$, i.e.,
\begin{align}
    & V_i = f_{V_i}(\vect{pa}(V_i), \vect{U}_{V_i}) \text{ where } f_{V_i} \in \Fcal,\ V_i \not\in \vect{pa}(V_i)
    \label{def:scm}
\end{align}
The only source of randomness are latent variables $\vect{U}$ with probability space $(\Omega, \Sigma, P_{\vect{U}})$. This induces a probability law over the observed variables $P_{\scm}$. We may omit the subscript $\scm$ and denote the observational probability distribution by $P$ throughout the text.
Given $\scm$, one can construct a graph with nodes $\vect{V} \cup \vect{U}$ and directed edges from nodes in $\vect{pa}(V_i) \cup \vect{U}_{V_i}$ to $V_i$. We call this graph the causal graph corresponding to SCM $\scm$ and denote it by $\graph_\scm$ or simply $\graph$. We assume $\graph$ is acyclic and known. Moreover, We will assume each node in the graph is $1$-dimensional.
For random variable $V$ in the SCM $\scm$, let $V_\scm(\vect{u})$ be its deterministic value after fixing a realization $\vect{u}$ of latent variables $\vect{U}$. The causal effect of treatment $T$ on target $Y$ is:
\begin{definition}[Causal Effect]
Let $Y_{\scm(T=t)}(\vect{u})$ be the value of $Y$ by fixing $\vect{U}=\vect{u}$ and changing the function $f_T$ to a constant function $f_T = t$ in $\scm$. Then, we call the random variable $Y_{\scm(T=t)}$ the causal effect of treatment $T=t$ on target $Y$. We may simplify the notation and write it as $Y_{t}$ if the SCM $\scm$ and treatment variable $T$ are known from context. Note that $Y_{T(\vect{u})}(\vect{u}) = Y(\vect{u})$.
\end{definition}
\vspace{-0.5em}
When $T$ is continuous, then we can view $\{Y_t: t \in supp(T)\}$ as a stochastic function defined on $(\Omega, \Sigma, P_{\vect{U}})$. This is referred to as the response function, partial dependence plot, and dose-response curve in the literature~\citep{zhao2021causal,ritz2015dose,chernozhukov2018automatic}.

\xhdr{Average treatment effect, average treatment derivative, and partial identification}
Our goal is to estimate bounds on the effectiveness of a treatment regime on a population from the observational
distribution $P$ and the causal graph $\graph$.
In the continuous treatment case, where there is no "on"/"off" notion of treatment, we can compare the average causal effect of an arbitrary treatment (dosage) to the average causal effect at a fixed point $T=t_0$. For example, to indicate the effect relative to not prescribing any treatment, we can choose $t_0 = 0$. This quantity is known as the average treatment effect, average level effect, or average dose effect in the literature on continuous treatment setting~\citep{hirano2004propensity,kennedy2017non,callaway2021difference}.
\begin{definition}[Average Treatment Effect]
    For SCM $\scm$, the average treatment effect (ATE) at $T=d$ w.r.t. a fixed point $T=t_0$ is defined as
    \begin{align}
        \text{ATE}_{\scm}(d) := \E_{\vect{u}\sim P_{\vect{U}}}[Y_{\scm(T=d)}(\vect{u}) - Y_{\scm(T=t_0)}(\vect{u})]
        \label{eq:ate}
    \end{align}
\end{definition}
Note that estimating the ATE and finding bounds on it only depends on the value of the average response function in $T=d$ and $T=t_0$. As pointed in~\citet{gunsilius-iv}, this quantity can take arbitrary values if we do not make any assumptions on the set of response functions. Here, we assume the partial derivative of the response function w.r.t. the treatment, i.e., ${\partial Y_t}/{\partial t}$ exists and is a bounded continuous function. We then define the average treatment derivative as the following:
\begin{definition}[Average Treatment Derivative]
For the treatment regime $f_T$ in SCM $\scm$, we define the average treatment derivative (ATD) as 
\begin{align}
    \text{ATD}_{\scm} = \E_{\vect{u} \sim P_{\vect{U}}}\left[\frac{\partial Y_{\scm(T=t)}(\vect{u})}{\partial t}\Big|_{t=T(u)}\right],
    \label{eq:atd}
\end{align}
\end{definition}
Estimating the ATD can be seen as a proxy for the effectiveness of the prescribed treatment, where we consider the population-level average effect of an infinitesimal increase in the treatment/dosage~\citep{rothenhausler2019incremental}. In this work, however, we leverage the regularity of this quantity to achieve smoother solutions to the ATE estimation. We will expand on this in~\autoref{sec:method}.

Note that we cannot readily use eq.~\ref{eq:ate} (or eq.~\ref{eq:atd}) to estimate the ATE (or ATD), as we only have access to the observational
distribution $P$ and the causal graph $\graph$ and not the latent distribution $P_\vect{U}$.
In fact, ATEs are generally non-identifiable, i.e., there exist multiple SCMs with the same causal graph $\graph$ and
generated distribution $P$ that result in different values of ATE.
For some graphs, however, one can use non-parametric identification algorithms like $do$-calculus to identify the causal effect from the
observational distribution~\citep{causal-pearl}. In practice, even for identifiable causal graphs, we cannot pinpoint the true ATE due to the uncertainty caused by sampling variation and finite sample errors. Instead, we are interested in finding a tight set of possible solutions that will contain the true value of ATE (or ATD) with high probability. 
This is known as the partial identification problem in the literature~\citep{manski2003partial}.\footnote{\small In the literature, partial identification is not concerned with sampling uncertainty and is defined population-wise for non-identifiable causal effects. However, in this paper, we abuse the terminology and use partial identification for both non-identifiable quantities and identifiable effects with finite samples.} More formally, the partial identification of ATDs/ATEs is defined as:
\begin{definition}[Partial Identification of ATD/ATE]
    \label{def:partial-id}
    Partial identification of ATD is the solution to the following optimization problem:
    \begin{equation}
    (\min_{{\scm'} \in \mathfrak{M}} \text{ATD}_{\scm'},\max_{{\scm'} \in \mathfrak{M}} \text{ATD}_{\scm'})\text{ s.t. }P_{\scm'} = P\text{ \& }\graph_{\scm'} = \graph
        \label{opt:PI}
    \end{equation}
    where $\mathfrak{M}$ is the set of all SCMs on random variables $\vect{V}$.
    We denote the solution to the above problem as $(\lowerATD, \upperATD)$.
    Similarly, we can define the partial identification of ATEs by replacing $\text{ATD}_{\scm'}$ with $\text{ATE}_{\scm'}(d)$
    in eq.~\ref{opt:PI}.
    We refer to the solution to the latter problem as $(\lowerATE(d), \upperATE(d))$.
\end{definition}

\xhdr{Implicit generative models}
To solve the partial identification problem, we use the expressive power of generative models to satisfy the distribution constraint in eq.~\ref{opt:PI}. Choosing distance measures such as Jensen-Shannon divergence or Wasserstein metric results in models such as GANs or Wasserstein GANs (WGANs)~\citep{goodfellow2014generative, wgan}.
The typical way to implement these models is to solve a minimax game between the generator and a discriminator.
However, adding the ATD minimization/maximization term to the minimax loss function will result in unstable training.
Instead, in our practical algorithm, we will use Sinkhorn Generative Networks (SGNs) that use Sinkhorn divergence $S_{\epsilon}$, a differentiable $\epsilon$-approximation of
Wasserstein metric, as the distance measure between generated and true samples~\citep{cuturi2013sinkhorn,sinkhorn2,sinkhorn}.
%
Due to the differentiability of Sinkhorn divergence, we will only need to train a generator network enabling us to sidestep much of the unstable minimax training in (W)GANs.

\section{Related Work}
This work builds upon partial identification and generative causal models.

\xhdr{Partial identification} Finding informative bounds on treatment effects has been well-studied 
in the existing literature (\citep{robin-bound,manski-bound,robin-graphical-ineq,ramsahai-bounds,richardson2014nonparametric,Miles2015OnPI,NoamFinkelstein,zhang2021bounding,zhang2021non}).
\citet{balke-bound} find the tightest possible bound for the discrete instrumental variable setting by converting it to
a linear programming problem.
For the backdoor setting and binary treatments, \citet{makermaggie} provide probabilistic upper/lower bounds on potential outcomes in the finite sample regime.
%
%
Recently, \citet{partial-ident-elias} and \citet{automated} independently describe a polynomial programming approach
to solve the partial identification for general causal graphs.
They both use the notion of canonical SCMs to map the latent variables to the space of functions from treatment $T$ to outcome $Y$.
Though they show their polynomial programming formulation finds the optimal bound, their approach is only applicable to
discrete random variables with small support.
In fact, the time complexity of their algorithm grows exponentially with the size of the support set of variables,
making their algorithm intractable for continuous settings.

\citet{gunsilius2019path} extends the commonly-used linear programming approach to partial identification of IV graphs with continuous treatments. 
They use a stochastic process representation of the variables and solve the linear programming via sampling.
However, their method suffers from stability issues, as discussed in \citet{iv-continuous} and is only applicable for the IV setting.
\citet{iv-continuous,padh2022stochastic} parameterize the space of response functions by assuming them as linear combinations of a set of fixed basis functions.
Then, they match the first two moments of observed distribution while minimizing/maximizing the ATE.
However, they do not provide any theoretical guarantees on the tightness of their derived bounds.

Most similar to our work is \citet{hu2021generative} who use generative adversarial networks (GANs) to match the observed distribution and search for response functions with maximum/minimum ATEs. They provide convergence guarantees for the instrumental variable causal graph with linear models. Their approach is also based on the minimax game between generator and discriminator, which can result in unstable training.
Our work differs in a few important ways. We focus on partial identification of average derivatives and use that to find bounds over the ATE. Using this approach, we show that our derived bounds converge to the optimal bounds for linear SCM with general causal graphs, including both identifiable and non-identifiable settings.
We use Sinkhorn divergence, a differentiable approximation of Wasserstein distance, to train our implicit generative models.
Empirically, we find that this avoids the unstable training of GANs.
\citet{guo2022partial} studied the partial identification of ATE with noisy covariates. 
Their work is similar to our approach in that we both use a similar robust optimization formulation. However, they focus on identifiable causal graphs, where one can use adjustment formulas such as the Back-door formula and make parametric assumptions on the joint distribution of observed variables. 

\xhdr{Generative causal models} 
\citep{goudet2017causal,yoon2018ganite,causalGAN,cntr-generative-networks} use generative models to capture a causal perspective on evaluating the effect of interventions on high-dimensional data such as images. They do not consider the problem of bounding treatment effects. 
\citet{causalneural} introduced Neural Causal Models (NCMs) that leverages the universal
approximability of neural networks to learn the SCM. Although it is not generally possible to learn the 
true SCM by training on the observational data, they prove that NCMs can be used to test the identifiability of causal effects and propose an algorithm to estimate identifiable causal effects.
Their work's theory and empirical instantiation are in the context of discrete random variable datasets.
Our work builds upon NCMs for partial identification with both continuous and discrete random variables.


\section{Partial Identification using Implicit Generative Models}
\label{sec:method}
We explain our method to solve the partial identification problem in Def.~\ref{def:partial-id} using implicit generative models.
In~\autoref{subsec:graph-generative}, we describe partial identification of ATDs as a constrained optimization problem using $\graph$-constraint generative models~\citep{causalneural}.
Then, in~\autoref{subsec:theory}, we show that the solution to this constrained optimization problem converges to the optimal bounds on the ATD in infinite data samples.
We prove our results for linear SCMs with general causal graphs, i.e., both identifiable and non-identifiable settings.
Next, we propose our approach to extend the partial identification of ATDs to ATEs.
Finally, we describe a practical algorithm to solve our distributionally-constrained optimization problem in~\autoref{subsec:algorithm}.
\subsection{$\graph$-constraint generative models}
\label{subsec:graph-generative}
To solve the partial identification problem, we need to search over the set of all possible SCMs $\mathfrak{M}$.
This is generally not feasible, as there is no constraint on the distribution of the latent variables $P_{\vect{U}}$,
as well as the function family $\mathcal{F}$.
Instead, we parameterize the space of all SCMs that are consistent with causal graph $\graph$ using neural networks.
More specifically, we use $\graph$-constraint generative models:
\begin{definition}[$\graph$-constraint Generative Models (Def.~7 in~\citet{causalneural})]
    For a given causal graph $\graph$, a $\graph$-constraint generative model is a tuple
    $\Mcal_\graph^\theta = (\vect{V}, \hat{\vect{U}}, \Fcal^\theta, \hat{P}_{\hat{\vect{U}}})$, where each $V_i \in
    \vect{V}$ is generated from
    \begin{align}
        V_i = f^\theta_{V_i}(\vect{pa}(V_i), \hat{\vect{U}}_{\vect{C}}) \text{ for }  f^\theta_{V_i} \in \Fcal^\theta,
    \end{align}
    where $\vect{pa}(V_i)$ is the observed parents of node $V_i$ in $\graph$ and $\hat{\vect{U}}_{\vect{C}}\in \vect{U}$ is the latent noise corresponding to maximal $C^2$-Component $\vect{C} \subseteq \vect{V}$ containing node $V_i$, i.e., each pair of variables in $\vect{C}$ have common latent parent nodes.
    In addition, $\hat{P}_{\hat{U}} \sim \texttt{Unif}(0, 1)$ for each $\hat{U} \in \hat{\vect{U}}$.
    \label{def:graph-generative}
\end{definition}
$\graph$-constraint generative models make the search over the set $\mathfrak{M}$ feasible by limiting it to generative
models with \textit{uniformly} distributed latent variables that are consistent with causal graph $\graph$.
In fact, in their Theorem~3, \citet{causalneural} show that for \textit{any} discrete SCM $\scm^*$ with causal
graph $\graph$, there exists a $\graph$-constrained generative model $\Mcal_\graph^\theta$ that generates the same causal
effect, i.e., $Y_{\scm^*(T=t)} = Y_{\Mcal_\graph^\theta(T=t)}$ a.s. Their proof technique, however, only works for SCMs with discrete variables.
Here, we do not prove the expressiveness of $\graph$-constrained generative models for continuous SCMs. Instead, for
simplicity and completeness of our theoretical results in \autoref{subsec:theory}, we assume that the true SCM is
a $\graph$-constrained generative model itself.
In our experiments, we empirically show that our results hold even for SCMs with different latent distributions, such as Gaussian noise.
\begin{assump}
    The true SCM $\scm$ is a $\graph$-constrained generative model.
    In other words, there exist $\theta$ such that $\scm = \Mcal_\graph^\theta$.
    \label{assump:true-scm}
\end{assump}

Under Assumption~\ref{assump:true-scm}, we reformulate the problem in eq.~\ref{opt:PI}
using generative models, i.e.,
\begin{equation}
    (\min_{\theta} \text{ATD}_{\scm^\theta_\graph},\;
        \max_{\theta} \text{ATD}_{\scm^\theta_\graph})\text{ s.t. }\ P_{\scm_{\graph}^{\theta}} = P
    \label{opt:new-PI}
\end{equation}

In practice, we never have access to true distribution $P$ as we only observe a finite number of samples corresponding to the empirical distribution
$P^n = \frac{1}{n} \sum_{i=1}^n \delta_{\vect{v}^{(i)}}$ for a given dataset
$\{\vect{v}^{(1)}, \cdots, \vect{v}^{(n)}\}$.
Also, the observed variables may be biased due to noisy measurements.
Therefore, we reformulate the problem in eq.~\ref{opt:new-PI} as a constrained optimization problem.
We choose the 1-Wasserstein metric as our distance measure, which naturally results in generative models such as WGANs.
We will state our theory in~\autoref{subsec:theory} based on this metric.
However, in~\autoref{subsec:algorithm}, we will propose a practical algorithm that uses Sinkhorn divergence, a differentiable approximation of 1-Wasserstein distance, for more stable results.
Our constrained optimization problem is as follows:
\begin{equation}
    \left(\min_{\theta} \text{ATD}_{\scm^\theta_\graph},\max_{\theta} \text{ATD}_{\scm^\theta_\graph}\right) \text{ s.t. }\ W_1\left(P_{\scm_{\graph}^{\theta}}, P^n\right) \leq \alpha_n
    \label{opt:robust-PI}
\end{equation}
where $\alpha_n$ is a hyper-parameter that specifies the level of tightness of the bounds.
We denote the solution to eq.~\ref{opt:robust-PI} as $(\hat{\lowerATD}, \hat{\upperATD})$.
In the case of noisy measurements, we need domain knowledge of how noisy the data is to determine the value of $\alpha_n$.
Otherwise, we can use the finite-sample convergence rate of empirical Wasserstein distance to choose an appropriate value of $\alpha_n$~\citep{weed2019sharp}.
As our theoretical results are concerned with the infinite-sample case, we will assume that there exist values of $\alpha_n$ such that the true distribution lies within the Wasserstein ball.

\begin{assump}
    For each $n \in \mathbb{N}$, there exist $\alpha_n > 0$ such that $W_1(P, P^n) \leq \alpha_n$.
    \label{assump:wass-ball}
\end{assump}

\subsection{Theoretical guarantees and extension to ATEs}
\label{subsec:theory}
Assumptions~\ref{assump:true-scm} and~\ref{assump:wass-ball} ensure that the bound derived by eq.~\ref{opt:robust-PI}
contains the true value of ATD.
However, we do not know how informative/tight the derived bounds are.
In fact, one can always return $(-\infty, +\infty)$ as one solution to partial identification.
This part gives theoretical guarantees that our algorithm can result in tight bounds over ATD.
In particular, we focus on linear SCMs and show that, under the infinite number of samples, our algorithm
converges to the optimal bound $(\lowerATD, \upperATD)$ for both identifiable and non-identifiable causal graphs. See~\autoref{appx:proof-tight-bound} for the proof.
%
\begin{definition}[Linear SCMs]
    SCM $\scm = (\vect{V}, \vect{U}, \mathcal{F}, P_{\vect{U}})$ is linear, if
    \begin{align}
        V_i = \vect{a}^\top_{V_i} \vect{pa}(V_i) + \vect{b}^\top_{V_i} \vect{U}_{V_i} \;\text{ for vectors }\vect{a}_{V_i}, \vect{b}_{V_i} \in \mathcal{F}
    \end{align}
\end{definition}

\begin{theorem}[Tight Bounds]
    Assume the dataset $\{\vect{v}^{(1)}, \cdots, \vect{v}^{(n)}\}$ is generated from a linear SCM.
    Then, under assumptions~\ref{assump:true-scm}, and~\ref{assump:wass-ball}, the solution
    to the constrained optimization problem in eq.~\ref{opt:robust-PI} converges to the optimal bound over the ATD in infinite samples,
    i.e., $\hat{\lowerATD} \to \lowerATD$ and $ \hat{\upperATD} \to \upperATD$.
    \label{theorem:tight-bound}
    \vspace{-1em}
\end{theorem}
Up until now, we have only focused on partial identification of ATDs.
Here, we discuss how to extend our results to find bounds on ATEs.
A naive solution is to replace ATD with ATE in eq.~\ref{opt:robust-PI} and directly optimize it. However, as demonstrated in the experiments, this approach can result in non-informative bounds. In fact, ATE with continuous treatments is a non-regular quantity~\citep{kennedy2017non}.

Instead, we claim that we can use the same generative model trained for partial identification of ATD to bound the value
of ATE.
In particular, we define a new objective function by uniform intervention on the treatment, which we call
uniform average treatment derivative (UATD), and show that the solution to partial identification of UATD matches the solution to partial
identification of ATEs.
\begin{definition}[UATD] \label{def:uatd}
    For an SCM $\scm$, we define the uniform average treatment derivative (UATD) at interval $[t_0, d]$ as
    \begin{align}
        \text{UATD}_{\scm}[t_0, d] := \E_{\vect{u}\sim P_{\vect{U}}}
        \left[\E_{t \sim \texttt{Unif}[t_0, d]}\left[\frac{\partial Y_t(\vect{u})}{\partial t}\right]\right]
    \end{align}
\end{definition}
Now, we state our result on using average derivatives to solve partial identification of ATEs:
\begin{corollary} \label{theorem:ade-theorem}
    Let $\theta^*$ be the solution to the partial identification of UATD at interval $[t_0, d]$.
    Then, $\theta^*$ is also a solution to the partial identification of $\text{ATE}(d)$. If the true SCM $\scm$ is linear, then the bound is tight, i.e., $\hat{\lowerATE}(d) \to \lowerATE(d)$ and $ \hat{\upperATE}(d) \to  \upperATE(d)$ as $n \to \infty$.
\end{corollary}
The proof is given in \autoref{appx:proof-ade}. 

\xhdr{Remark} ATD and ATE are generally different quantities and finding bounds on one does not necessarily results in bounds on the other. In fact, ATE is not pathwise differentiable for continuous treatments~\citep{diaz2013targeted,kennedy2017non,chernozhukov2018automatic} (See~\autoref{appx:ate-atd}). We define UATD as a quantity to relate ATE and ATD. The formal definition of UATD in Def.~\ref{def:uatd} is the same as ATE up to a scale factor, and one would expect to see similar issues with ATE here as well. Therefore, we approximate UATD with a version that, instead of a uniform distribution of treatment within interval $[t_0, d]$ with zero density outside, the treatment distribution has continuous differentiable non-zero density defined over the \textit{whole} support of $T$. More concretely,
\begin{align}
    \text{ATE}_{\scm_{\graph}^{\theta}}(d)~\propto~ \text{UATD}_{\scm_\graph^\theta}[t_0, d] &= \E_{\vect{u} \sim P_{\hat{\vect{U}}}}\left[\int_{supp(T)} \frac{\partial Y_{\scm_\graph^\theta}(T=t)}{\partial t} d\mu(t) \right] \nonumber \\
    &  {\boldsymbol{\approx}}~ \E_{\vect{u} \sim P_{\hat{\vect{U}}}}\left[\int_{supp(T)} \frac{\partial Y_{\scm_\graph^\theta}(T=t)}{\partial t} d\tilde{\mu}(t) \right]
\end{align}
where $\mu$ is a the uniform measure within interval $[t_0, d]$ and $\tilde{\mu}$ is its approximation with a non-zero density over the full support of $T$. Choosing $\tilde{\mu}$ trades off between regularity of the response function and the approximation error. The more $\tilde{\mu}$ is close to the uniform measure $\mu$, the more irregularity we allow in the response curve and the less informative the bound will be, while the objective function is closer to ATE. Adding more density outside of interval $[t_0, d]$ imposes regularity on the response curve (Figure~\ref{fig:uatd-ate}). 
\begin{figure}[t]
    \centering
    \includegraphics{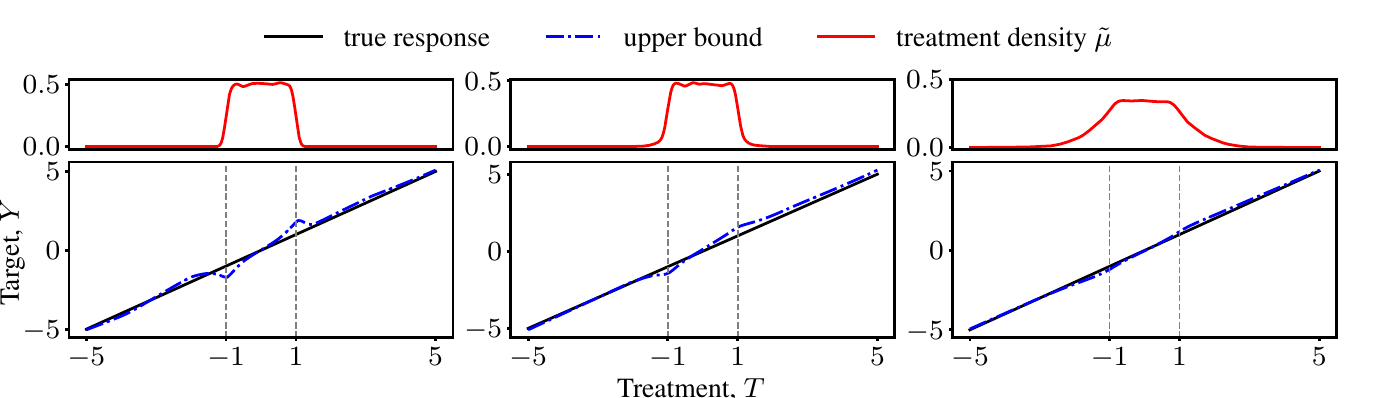}
    \caption{\small Maximizing the UATD between $T=-1$ and $T=1$ with approximation densities $\tilde{\mu}$. The standard deviation of $\tilde{\mu}$ increases from left to right, resulting in smoother response curves. All three plots are trained on a linear data generating process $Y = T + \mathcal{N}(0, 1)$ using $5{,}000$ samples}
    \label{fig:uatd-ate}
\end{figure}
It is important to note that we use the formal definition of UATD (with uniform treatment distribution) in Corollary~\ref{theorem:ade-theorem} since, in linear SCMs, it is not possible for the generator model to attain arbitrarily large values in points $t_0$ and $d$ as the derivative of linear functions remains fixed outside of interval $[t_0, d]$. This is why we state our theoretical results based on the uniform intervention.

\subsection{Our algorithm}
\label{subsec:algorithm}
We describe our algorithm to solve the optimization problem in eq.~\ref{opt:robust-PI}. 
We will focus on finding $\hat{\lowerATD}$, a similar approach can be taken for $\hat{\upperATD}$.
A general strategy is to convert the constrained problem to its unconstrained version using the method of Lagrange multiplier:
\begin{align}
    \min_{\theta} \max_{\lambda \geq 0} \text{ATD}_{\scm^\theta_\graph} +
    \lambda (W_1(P_{\scm_{\graph}^{\theta}}, P^n) - \alpha_n)
    \label{opt:lagrange}
\end{align}
As the Wasserstein distance is not differentiable, we cannot directly use gradient descent to solve eq.~\ref{opt:lagrange}. A common approach is to use the dual formulation of Wasserstein distance $W_1(P_{\scm_{\graph}^{\theta}}, P^n) = \max_{||q_{\phi}||_L\leq 1} \E_{P^n}[q_\phi(\vect{v})] - \E_{P_{\scm_{\graph}^{\theta}}}[q_\phi(\vect{v})]$ and solve eq.~\ref{opt:lagrange} using WGANs, a similar solution used in~\citet{hu2021generative}. However, this min-max-max formulation can result in unstable bounds as we show in our experiments.
Instead, we use Sinkhorn divergence, a differentiable approximation to Wasserstein distance, as the measure of distance between distributions and solve the following:
\begin{align}
    \min_{\theta} \max_{\lambda \geq 0} \text{ATD}_{\scm^\theta_\graph} +
    \lambda (S_\epsilon(P_{\scm_{\graph}^{\theta}}, P^n) - \alpha_n)
    \label{opt:lagrange-new}
\end{align}

To solve eq.~\ref{opt:lagrange-new}, we need to evaluate $\text{ATD}_{\scm_\graph^\theta}$ and calculate its gradient w.r.t. $\theta$.
As we are using $\graph$-constrained generative models, we can calculate the value of $Y_{\scm_\graph^\theta(T=t)}(\vect{u})$ by hard intervention $T=t$, i.e., fixing the output of function $f_T^\theta$ as $t$ and computing $Y$ through a topological order of calculations. 
Then, we estimate $\text{ATD}_{\scm_\graph^\theta}$ as follows:
\begin{align}
    \text{ATD}_{\scm_\graph^\theta} \approx \frac{1}{n} \sum_{i=1}^n \frac{1}{\epsilon} \left[Y_{\scm_\graph^\theta(T=t^{(i)}+\epsilon)}(\vect{u}^{(i)}) - Y_{\scm_\graph^\theta(T=t^{(i)})}(\vect{u}^{(i)})\right]
    \label{eq:calc-atd}
\end{align}
where $\{t^{(i)}\}_{i=1}^n$ are samples from the treatment variable, and $\{\vect{u}^{(i)}\}_{i=1}^n$ are the latent variables generated from a uniform distribution.
To choose an appropriate value of $\alpha_n$, we first train our generator without the ATD term until the Sinkhorn loss converges to some value and use that as our choice of $\alpha_n$. We then continue our training by adding the ATD term.

We note that using algorithms such as projected gradient descent to solve the constrained optimization problem requires us to project the weights of our network into the Wasserstein (Sinkhorn) ball in each step. This can be computationally infeasible, and current methods are mainly focused on special loss functions~\citep{mohajerin2018data,li2019first,wong2019wasserstein}. Instead, we consider an alternating optimization procedure, in which we alternate between updating the gradients for the ATD and the Sinkhorn loss. 
The full details of our algorithm, its extension to ATEs, and the alternating optimization are described in~\autoref{appx:algorithm}.

\section{Experiments}
We run our partial identification algorithms on a variety of simulated settings.
We mainly focus on the synthetic data generating processes as the ground truth must be known to evaluate our derived bounds properly.
\begin{figure}[t]
    \centering
    \includegraphics{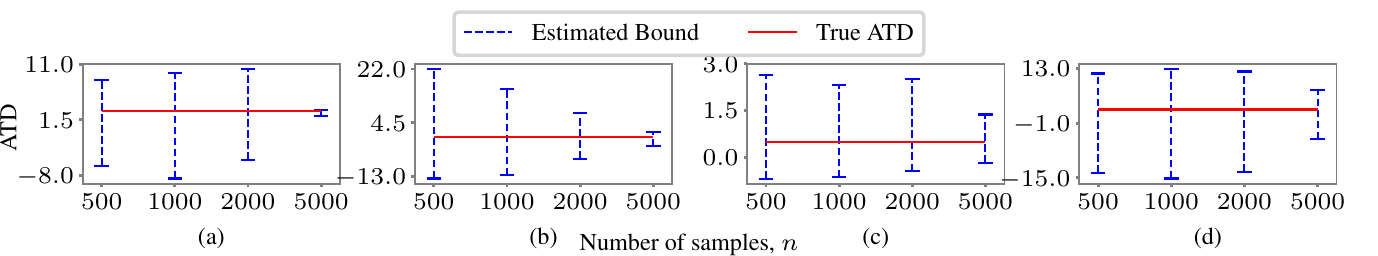}
    \caption{\small Our derived bounds on ATD for (a) linear Back-door, (b) Front-door, (c) linear IV and (d) leaky mediation settings. As the number of samples increases, our algorithm pin-points the ATD in identifiable settings and leads to tight bounds on it in non-identifiable cases.}
    \label{fig:atd}
\end{figure}
Our primary goal is to show that using uniform average treatment derivatives instead of directly optimizing the average treatment effect will result in tighter and more stable bounds. We first run our algorithm to estimate bounds on the value of ATDs for both identifiable and non-identifiable causal graphs. We show that, as the number of samples increases, our algorithm converges to tight bounds over the true value of ATD (Figure~\ref{fig:atd}).
We then focus on partial identification of ATEs
and demonstrate that using partial derivatives of the response function leads to more informative bounds, while being valid (Figures~\ref{fig:ATE},~\ref{fig:compare-flow}). 

Additionally, as a sanity check, we consider two binary-treatments datasets where the optimal bounds are known and show that our approach can reach the optimal solution. We also test our method on an ACIC dataset, a case study with real-world covariates, to illustrate the performance of our algorithm on higher-dimensional datasets. We include these additional experiments in~\autoref{appx:add-exp}. Our implementation details, as well as hyper-parameters can be found in~\autoref{appx:implementation}.
\begin{figure}[t]
    \centering
    \begin{subfigure}[b]{\linewidth}
    \centering
        \includegraphics[width=0.9\linewidth]{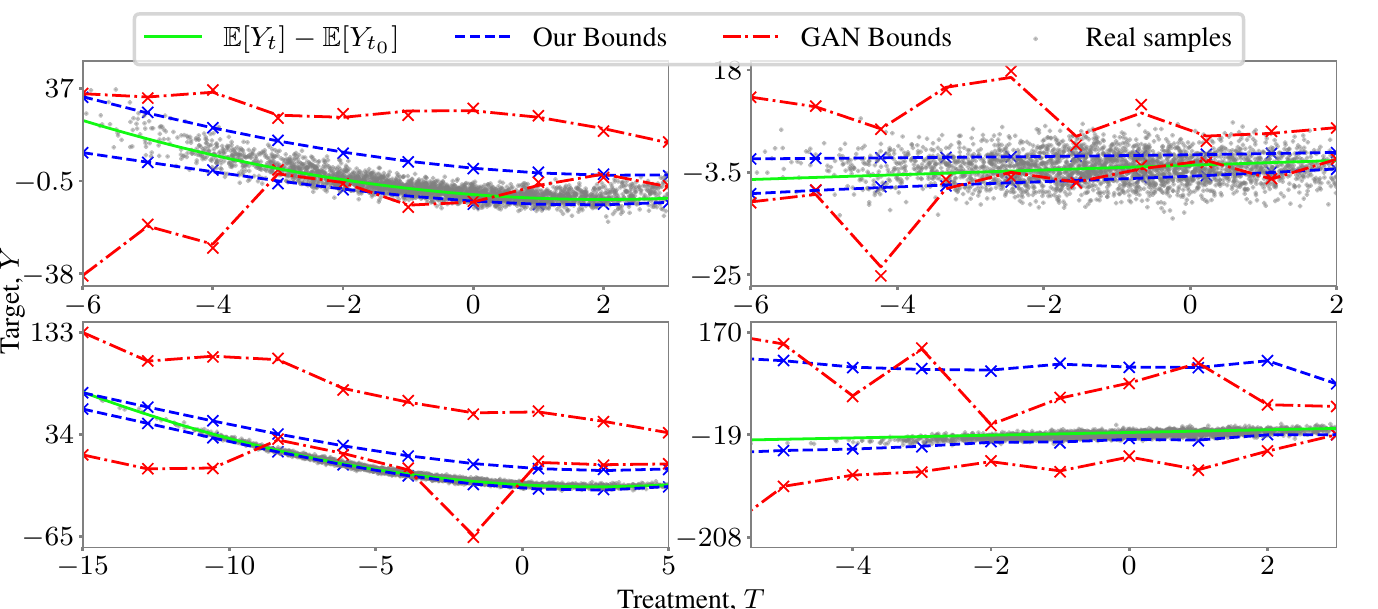}
    \caption{\small GAN baseline. (top left) Nonlinear Back-door, (top right) linear IV, (bottom left) nonlinear IV, and (bottom right) leaky mediation settings. $t_0$ is chosen as the maximum treatment value in each data. Our derived bounds are tighter and more stable than the GAN baseline, which directly optimizes the ATE.}
        \label{fig:ATE}
    \end{subfigure}

    ~\vspace{0.3em}
    
    \begin{subfigure}[b]{\linewidth}
        \centering
    \includegraphics[width=0.9\linewidth]{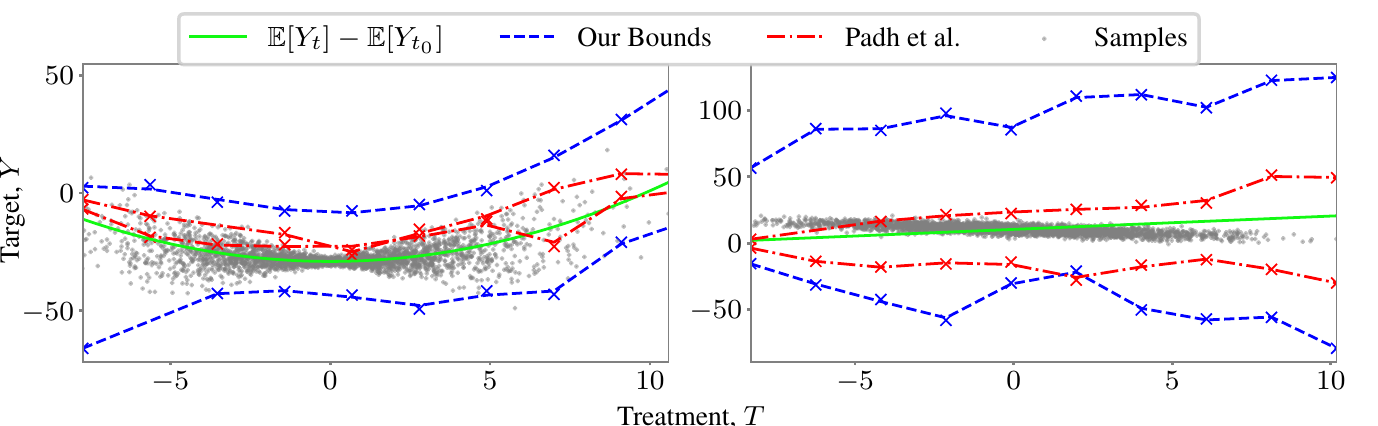}
    \caption{\small ~\citet{padh2022stochastic}. (left) Nonlinear IV, (right) linear IV with strong confounding. $t_0$ is chosen as the minimum treatment value in each data. While resulting in tighter bounds in the linear setting, the baseline leads to invalid bounds in the nonlinear dataset.}
    \label{fig:compare-flow}
    \end{subfigure}
    \vspace{0.1em}
    
    \caption{\small Comparing our results for partial identification of $\E[Y_{t}] - \E[Y_{t_0}]$ for $10$ different values of treatment.}
\end{figure}
\subsection{Datasets and Baseline}
\xhdr{Datasets}
 We consider various data generating processes for different causal graphs, including a linear SCM with three-dimensional covariates and a quadratic SCM with nonlinear interaction between the covariates and treatment for the Back-door causal graph (Figure~\ref{fig:backdoor-graph}), as well as a nonlinear SCM for the Front-door setting (Figure~\ref{fig:frontdoor-graph}).
 For the IV graph (Figure~\ref{fig:iv-graph}), we use four different SCMs, two linear and two nonlinear datasets, based on the strength of the instrument variable and the confounding. 
Finally, we generate a two-dimensional linear dataset for the leaky mediation causal graph (Figure~\ref{fig:leaky-graph}). The full details of our data-generating processes can be found in~\autoref{appx:DGPs}.
%


\xhdr{Baselines}
Our first baseline is the algorithm in~\citet{hu2021generative} that directly optimizes the value of ATE using GANs. 
We use their default hyper-parameters with a tolerance of $0.0001$.  
Similar to their experimental setup, we consider $50$ intermediate solutions where the distance is within the tolerance and compute the bounds using the mean and one-sided confidence intervals.

We also compare our algorithm with the moment-matching method in~\citet{padh2022stochastic}. They parameterize the first two moments of the generated distribution and match them with the observed samples. They choose response functions as linear combinations of a fixed number of basis functions. We consider the neural basis functions for our experiments with their default hyper-parameters. In particular, we train a 3-hidden layer MLP with 64 neurons in each layer to learn the target variable given the treatment. We then use the activation of the $k$th neuron in the last hidden layer as the $k$th basis function. Instead of maximizing/minimizing $E[Y_{T=t}]$, we consider $E[Y_{T=t}] - E[Y_{T=t_0}]$ as our goal is to bound the ATE between two points. We leave other implementation details untouched.

\subsection{Results}
\xhdr{Bounding average treatment derivatives} We generate data with sample sizes $N = \{500, 1000, 2000, 5000\}$ from the nonlinear Front-door and linear Back-door SCMs (identifiable), as well as linear IV with strong confounding and leaky mediation settings (non-identifiable).
We run our algorithm with ten different random seeds for each setting/sample size.
Then, we choose the five runs with the lowest tolerance parameter $\alpha_n$ and choose the upper (lower) bound as the maximum (minimum) value of the ATD within these five runs.
%
%
Figure~\ref{fig:atd} shows our derived bounds.
As expected, the algorithm is able to point-identify the value of ATD for identifiable scenarios as the number of samples increases (Figures~\ref{fig:atd} (a) and (b)).
In non-identifiable cases,
our algorithm leads to tight bounds containing the true value of ATD by increasing the number of samples as depicted in Figures~\ref{fig:atd} (c) and (d). This is in line with our results in Theorem~\ref{theorem:tight-bound}.

\xhdr{Bounding average treatment effects}
Here, we aim to demonstrate the effectiveness of using partial derivatives for bounding the ATE compared to the direct optimization approach.
We consider six different settings and run our algorithm for $10$ different values of treatment $\{t_i\}_{i=1}^{10}$ in each setting. We compute the value of ATE w.r.t. a fixed point $t_0$. For each value of $T$, we generate $N = 5{,}000$ samples and run each experiment five times. Then, we select the maximum (minimum) value of ATE within the five runs as the upper (lower) bound. We follow the same procedure for the baselines.

To find the bounds on ATE using our approach, we sample from a distribution with uniform density within $[t_0, t_i]$ and Gaussian tails outside of $[t_0, t_i]$, and maximize/minimize the partial derivatives. Figure~\ref{fig:ATE} shows the effectiveness of this approach in comparison to the GAN baseline. Our algorithm produces stable and tight bounds containing the true value of ATE, while the GAN baseline, which relies on the direct optimization of ATEs, results in unstable loose bounds that may not include the true value of the treatment effect. 

Figure~\ref{fig:compare-flow} compares our algorithm with the method in~\citet{padh2022stochastic}. First, we find that their method do not contain the actual response curve in the nonlinear IV setting (left). This is not a surprising result; since they consider a fixed set of basis functions, the set of possible response curves is more restricted than ours. If the basis functions are not carefully designed to capture the true form of the response curve, the resulting bounds can be invalid, thus not containing the actual value of ATE even if there are regimes where the approach provides a tighter fit. On the other hand, our approach considers a larger family of response functions, enabling us to both capture shape and variation in the response curve while bounding the effects. 
\section{Conclusion, Limitations and Future Work}
Our work introduces a novel method to estimate bounds on average treatment effects from observational data. Specifically, we propose optimizing the average treatment derivative, which in turn can be used to estimate the average treatment effect in treatment response curves. Empirically we find that the use of our method recovers known bounds on treatment effects in the discrete case and outperforms other methods based on implicit models for partial identification in the continuous case. 

There remain several limitations of this work. Our work builds on the constrained optimization problem defined by \cite{causalneural} instantiated in the context of the ATD. Developing new methods for function maximization/minimization approaches under distributional constraints remains an important direction for future work. Our work primarily uses carefully designed synthetic datasets to evaluate our method under different constraints on the data distribution. A larger-scale evaluation of our approach on real-world benchmarks will better help us assess the method’s practicality. Moreover, the theory of our work is restricted to the linear SCM scenario. We have also made regularity assumptions throughout the paper, including Assumption 1, that the true SCM can be modeled using implicit generative models with uniform confounding distribution, as well as the approximation of UATD with regular treatment distributions. Consequently, practitioners must exercise caution when
deploying this method when there are nonlinear or irregular structures among the random variables. Finally, we have stated our results (Theorem 1 and Corollary 1) without quantifying the finite-sample estimation uncertainty. One can use the existing theory of finite-sample convergence of empirical Wasserstein distance in \cite{weed2019sharp} to extend our results with high-probability bounds.

\section*{Acknowledgments} 
We thank Tom Ginsberg and Michael Cooper for many helpful discussions. We also thank David Alvarez Melis for his suggestion on using Sinkhorn divergence instead of Wasserstein distance.
This research was supported by NSERC Discovery Award RGPIN-2022-04546 and a CIFAR AI Chair. Resources used in preparing this research were provided, in part, by the Province of Ontario, the Government of Canada through CIFAR, and companies sponsoring the Vector Institute.

\clearpage
\bibliography{main.bib}

\clearpage
\textbf{Paper checklist}

\begin{enumerate}
\item 
\begin{enumerate}
   \item Do the main claims made in the abstract and introduction accurately reflect the paper's contributions and scope?
    \answerYes{} \emph{We have listed the key contributions of our work in the introduction. In the experimental section we verify our theoretical contributions and perform a comparison to the closest baseline. Our theoretical results are in  idealized linear settings though we do experiment beyond that in our work. We highlight the limitations of our theoretical results and explain the rationale for the approximations that we make in practice. }
  \item Did you describe the limitations of your work?
    \answerYes{} \emph{We explicitly highlight the limitations of our work in the context of related work, in our exposition of the theoretical results and the computation necessary for our method. Our theoretical results make explicit the assumptions that we make and full proofs are presented in the supplementary material.}
  \item Did you discuss any potential negative societal impacts of your work?
    \answerYes{} \emph{Our paper explicitly tackles the problem of bounding treatment effects from observational data. We have highlighted that our theory is focused on the case of linear models and provided empirical results supporting the method's utility in the general, non-linear case. If the method is deployed in high-risk settings such as healthcare, there may  be times where the method fails to estimate bounds correctly from real-data in non-linear settings. We have encouraged practitioners to excercise caution in this regard.}
  \item Have you read the ethics review guidelines and ensured that your paper conforms to them?
    \answerYes{}
\end{enumerate}

\item If you are including theoretical results...
\begin{enumerate}
  \item Did you state the full set of assumptions of all theoretical results?
    \answerYes{}
        \item Did you include complete proofs of all theoretical results?
    \answerYes{} \emph{We include the complete proofs of our results in the supplementary material.}
\end{enumerate}

\item If you ran experiments...
\begin{enumerate}
  \item Did you include the code, data, and instructions needed to reproduce the main experimental results (either in the supplemental material or as a URL)?
    \answerYes{}
  \item Did you specify all the training details (e.g., data splits, hyperparameters, how they were chosen)?
    \answerYes{} \emph{ Our supplemental material makes clear the training procedure for our method, the main sets of hyperparameters and any heuristics for optimizing the models. We will release code to reproduce our results.  Our work made no use of participants. }
        \item Did you report error bars (e.g., with respect to the random seed after running experiments multiple times)?
    \answerYes{} \emph{Our method finds maximum/minimum values on the quantity of interest. Instead of reporting error bars, we consider maximum/minimum values among multiple runs with different random seeds.}
        \item Did you include the total amount of compute and the type of resources used (e.g., type of GPUs, internal cluster, or cloud provider)?
    \answerYes{} \emph{Experiments were run on two a single machine with 2 Nvidia RTX A5000 GPUs, 2 RTX A6000 cards GPUs, 96 CPUs and 188GB of RAM and on a larger internal cluster comprising hundreds of GPUs.}
\end{enumerate}

\item If you are using existing assets (e.g., code, data, models) or curating/releasing new assets...
\begin{enumerate}
  \item If your work uses existing assets, did you cite the creators?
    \answerYes{}
  \item Did you mention the license of the assets?
    \answerYes{} \emph{Our evaluation is primarily on synthetic data for which we have cited the original authors who created them.}
  \item Did you include any new assets either in the supplemental material or as a URL?
    \answerNo{}
  \item Did you discuss whether and how consent was obtained from people whose data you're using/curating?
    \answerNA{} 
  \item Did you discuss whether the data you are using/curating contains personally identifiable information or offensive content?
    \answerNA{} \emph{We use publicly available datasets that do not contain unanonymized personal data.}
\end{enumerate}

\end{enumerate}

\newpage
\appendix

\section{Proof of Theorem~\ref{theorem:tight-bound}} \label{appx:proof-tight-bound}
Here, we only focus on the minimization problem, i.e., $\hat{\lowerATD} \to \lowerATD$.
The proof of $\hat{\upperATD} \to \upperATD$ will similarly follow.
\subsection*{Assumptions}
Before stating our proof, we list the assumptions needed as follows. These are mainly technical assumptions that will simplify the derivations.
\begin{enumerate}
    \item In this proof we assume all random variables are one-dimensional.~\footnote{For high-dimensional variables, if the node-level causal graph is given, i.e., the relation between each dimension of variables, we can convert each multi-dimensional variable to multiple one-dimensional ones and follow the same proof. If the node-level causal graph is unknown and there is no inter-dependence between each dimension, one can follow the same proof technique by assuming $V_i$ as a vector in eq.~\ref{eq:linear-scm-param}. We leave these extensions as future work.}
    \item Since the true SCM $\scm$ is $\graph$-constrained (Assumption~\ref{assump:true-scm}) and linear, we will consider linear $\graph$-constrained SCMs as our parameter search space. More concretely, each random variable $V_i$ in an SCM $\scm_\graph^\theta$ can be written in the following form:
    \begin{align}
        V_i = \bm{\theta}^\top_{V_i} \vect{pa}(V_i) + \hat{\bm{\theta}}^\top_{V_i} \hat{\vect{U}}_{\vect{C}_i}
        \label{eq:linear-scm-param}
    \end{align}
    where $\bm{\theta}_{V_i} \in \R^{|\vect{pa}(V_i)|}$, $\hat{\bm{\theta}}_{V_i} \in \R^{|\hat{\vect{U}}_{\vect{C}_i}|}$, and $\theta \in \Theta$ is the concatenation of all $\bm{\theta}_{V_i}, \hat{\bm{\theta}}_{V_i}$ in a topological order of $V_i$s. Note that $\theta \in \R^{|E|}$, where $E$ is the set of edges in the causal graph $\graph$, containing edges to both observed and unobserved random variables.
    \item We assume the set of feasible parameters $\Theta$ is a bounded closed subset of $\R^{|E|}$. In practice, even in non-identifiable cases with infinite bounds, we use regularization to ensure the parameters of the network are bounded.
    \item The induced probability over observed random variables $P_{\scm_\graph^\theta}$, which we will write as $P_\theta$, and the true distribution $P$ belong to the Wasserstein space of order $p=1$, which we refer to as $\mathcal{P}_1$. In other words, $\int_{\R^d} |\vect{v}_0 - \vect{v}|\, dP_\theta(\vect{v}) < +\infty$ for any arbitrary point $\vect{v}_0 \in \R^d$. Again, in practice, since the support of all observed random variables is bounded, all probability distributions defined on them belong to the Wasserstein space.
    
\end{enumerate}
Before the main proof, we will state and prove the following useful lemmas.
\begin{lemma}
    Let define the set of feasible parameters for $n$ number of samples as $\Theta^n = \{\theta \in \Theta;\; W_1(P_{\scm_\graph^\theta}, P^n) \leq \alpha_n\}$.
    Then,
    $P_{\scm_\graph^{\theta^n}} \to P$ weakly for every sequence of $\theta^n \in \Theta^n$.
    \label{lemma:convergence}
\end{lemma}
\begin{proof}
    First, note that the empirical distribution $P^n$ weakly converges to $P$ as $n \to \infty$. Since $W_1$ metrizes $\mathcal{P}_1$, we have $W_1(P^n, P) \to 0$~\citep{villani2009optimal}. 
    Hence, we can choose the sequence of $\alpha_n$ for the distributional constraint, such that $W_1(P, P^n) \leq \alpha_n$ and $\alpha_n \to 0$ as $n \to \infty$.
    For any parameter $\theta^n \in \Theta^n$, we have $W_1(P_{\scm_\graph^{\theta^n}}, P^n) \leq \alpha_n\}$. Therefore, $W_1(P_{\scm_\graph^{\theta^n}}, P^n) \to 0$. Using the triangle inequality, we have
    \begin{align}
       W_1(P_{\scm_\graph^{\theta^n}}, P) \leq W_1(P_{\scm_\graph^{\theta^n}}, P^n)  + W_1(P^n, P)
    \end{align}
    Therefore, $W_1(P_{\scm_\graph^{\theta^n}}, P) \to 0$ or equivalently, $P_{\scm_\graph^{\theta^n}} \to P$ weakly.
\end{proof}
\begin{lemma}
    For any linear $\graph$-constrained SCM $\scm_\graph^\theta$, we have
    \begin{align}
        \vect{V}_\theta = \vect{A}(\theta) \hat{\vect{U}}
    \end{align}
    where each element in the matrix $\vect{A}(\theta) \in \R^{d \times |dim(\hat{\vect{U}})|}$ is a continuous function of $\theta$.
    \label{lemma:continuous-func}
\end{lemma}
\begin{proof}
    Consider a topological order of observed variables $(V_1, \cdots, V_d)$. We prove the result by induction. For $V_1$, from eq.~\ref{eq:linear-scm-param}, we have
    \begin{align}
        V_1 = 0 + \hat{\bm{\theta}}^\top_{V_1} \hat{\vect{U}}_{\vect{C}_1}=  {\bm{\phi}}^\top_{V_1} \hat{\vect{U}}
    \end{align}
    where the elements of ${\bm{\phi}}_{V_1}$ matches $\hat{\bm{\theta}}_{V_1}$ for $\hat{\vect{U}}_{\vect{C}_1}$ and equal to zero for $\hat{\vect{U}}\backslash \hat{\vect{U}}_{\vect{C}_1}$. Now, assume all variables ${V_1, \cdots, V_{d-1}}$ can be written as $\bm{\phi}_{V_i}^\top \hat{\vect{U}}$, where $\bm{\phi}_{V_i}$ is a continuous function of $\theta$. Then, from eq.~\ref{eq:linear-scm-param},
    \begin{align}
        V_d = \bm{\theta}^\top_{V_d} \vect{pa}(V_d) + \hat{\bm{\theta}}^\top_{V_d} \hat{\vect{U}}_{\vect{C}_d} = \bm{\theta}^\top_{V_d} (\bm{\phi}_{pa_1(V_d)}^\top \hat{\vect{U}}, \cdots, \bm{\phi}_{pa_r(V_{d})}^\top \hat{\vect{U}})^\top + \hat{\bm{\theta}}^\top_{V_d} \hat{\vect{U}}_{\vect{C}_d} = \bm{\phi}_{V_d}^\top \hat{\vect{U}}
    \end{align}
    where $\bm{\phi}_{V_d}$ is a linear function of $\bm{\theta}_{V_d}, \bm{\phi}_{V_1}, \cdots, \bm{\phi}_{V_{d-1}}$, and $\hat{\bm{\theta}}_{V_d}$. Defining matrix $\vect{A}(\theta)$ with rows $ \bm{\phi}_{V_i}$ concludes the proof.
\end{proof}
Now, we are ready to prove the main result. \begin{theorem}[Tight Bounds]
    Assume the dataset $\{\vect{v}^{(1)}, \cdots, \vect{v}^{(n)}\}$ is generated from a linear SCM.
    Then, under assumptions~\ref{assump:true-scm}, and~\ref{assump:wass-ball}, the solution
    to the constrained optimization problem in eq.~\ref{opt:robust-PI} converges to the optimal bound over the ATD in infinite samples,
    i.e., $(\hat{\lowerATD}, \hat{\upperATD}) \to (\lowerATD, \upperATD)$.
\end{theorem}
\begin{proof}
    The goal is to show the solution to 
    \begin{equation}
        \min_{\theta \in \Theta} \text{ATD}_{\scm^\theta_\graph}\text{ s.t. }\ W_1\left(P_{\scm_{\graph}^{\theta}}, P^n\right) \leq \alpha_n
    \end{equation}
    converges to the solution to 
    \begin{equation}
        \min_{\theta \in \Theta} \text{ATD}_{\scm^\theta_\graph},\;
           \text{ s.t. }\ P_{\scm_{\graph}^{\theta}} = P
    \end{equation}
    as $n \to \infty$.
    We first aim to re-write the value of $\text{ATD}$.
    Note that, in a linear $\graph$-constrained SCM, the partial derivative $\frac{\partial Y_{\scm_\graph^\theta(T=t)}(\vect{u})}{\partial t}$ is only a function of SCM parameters $\theta$. 
    Let define $g(\theta) = \frac{\partial Y_{\scm_\graph^\theta(T=t)}(\vect{u})}{\partial t}$.
    Then,
    \begin{align}
        \text{ATD}_{\scm_\graph^\theta}
        &= \int_{\Omega}  \frac{\partial Y_{\scm_\graph^\theta(T=t)}(\vect{u})}{\partial t} \Big|_{t = T(\vect{u})}\, dP_{\hat{\vect{U}}}(\vect{u})
        = \int_{\Omega} g({\theta})\, dP_{\hat{\vect{U}}}(\vect{u}) = g(\theta)
    \end{align}
    Therefore, we need to show the following to conclude the proof:
    \begin{align}
        \min_{\theta^n \in \Theta^n} g(\theta^n) \to
        \min_{\theta^\infty \in \Theta^\infty} g(\theta^{\infty})
    \end{align}
    where 
    \begin{align}
        \Theta^n &= \{\theta \in \Theta;\; W_1(P_{\theta}, P^n) \leq \alpha_n\} \nonumber \\
        \Theta^\infty &= \{\theta \in \Theta;\; P_{\theta} = P\} = \{\theta \in \Theta;\; W_1(P_{\theta}, P) = 0\}
    \end{align}

    Since $\Theta^\infty \subseteq \Theta^n$ for each $n \in \mathbb{N}$, we know that
    \begin{align}
    \min_{\theta^n \in \Theta^n} g(\theta^n) \leq
        \min_{\theta^\infty \in \Theta^\infty} g(\theta^{\infty})  
    \end{align}
    It is sufficient to show that, for each $\epsilon > 0$, there is $n_0$ such that for all $n > n_0$ we have
    \begin{align}
    \min_{\theta^n \in \Theta^n} g(\theta^n) \geq 
        \min_{\theta^\infty \in \Theta^\infty} g(\theta^{\infty}) - \epsilon   
    \end{align}
    Suppose this is not true, i.e., there exists $\epsilon > 0$ such that for each $n \in \mathbb{N}$ we have
    \begin{align}
        \min_{\theta^n \in \Theta^n} g(\theta^{{n}}) < g(\theta^{\infty}) - \epsilon
    \end{align}
    for all $\theta^\infty \in \Theta^\infty$. Let $\theta_\star^n = \arg\min_{\theta \in \Theta^n} g(\theta)$. The sequence $(\theta_\star^n)_{n\in \mathbb{N}}$ is a subset of $\Theta$ and therefore is a bounded sequence in $\mathbb{R}^{|E|}$. Thus, by Bolzano-Weierstrass theorem, there exists a convergent sub-sequence $(\theta_\star^{n_i})_{i \in \mathbb{N}}$ that converges to some fixed parameter $\theta_0$. Also, $\theta_0 \in \Theta$ as $\Theta$ is closed. Now, since $g$ is continuous, we have
    \begin{align}
        g(\theta_0) = \lim_{i \to \infty} g(\theta_\star^{n_i}) \leq g(\theta^\infty) - \epsilon
    \end{align}
    for all $\theta^\infty \in \Theta^\infty$. Hence, $\theta_0 \not \in \Theta^\infty$.
    
    On the other hand, using Lemma~\ref{lemma:continuous-func}, we have $\vect{V}_{\theta_\star^{n_i}} = \vect{A}(\theta_\star^{n_i}) \hat{\vect{U}}$. Since $\vect{A}(\theta_\star^{n_i})$ is a continuous function of $\theta_\star^{n_i}$, from the continuous mapping theorem, we have
    $P_{\theta_\star^{n_i}} \to P_{\theta_0}$ weakly. Also, from Lemma~\ref{lemma:convergence}, we have $P_{\theta_\star^{n_i}} \to P$. Therefore, $P = P_{\theta_0}$. Since $\theta_0 \in \Theta$, we conclude that $\theta_0 \in \Theta^\infty$, a contradiction.
\end{proof}
\xhdr{Remark} In this proof, we did not separate identifiable and non-identifiable cases. In fact, the notion of identifiability can be seen as a property of the set of feasible parameters $\Theta^n$ and $\Theta^\infty$. For example, in identifiable cases, we expect the set $\Theta^\infty$ to only contain one element while in non-identifiable cases, it consists of multiple possible solutions. Our proof holds as long as $\Theta^n$ and $\Theta^\infty$ are bounded subsets of $\R^{|E|}$.

\section{Proof of Corollary \ref{theorem:ade-theorem}} \label{appx:proof-ade}
Similar to the proof of Theorem~\ref{theorem:tight-bound}, define the set of feasible parameters as $\Theta^n = \{\theta\in \Theta \;\; W_{1}(P_{\scm_\graph^\theta}, P^n) \leq \alpha_n\}$.
Then,
\begin{align}
    \text{ATE}_{\scm_\graph^\theta}(d) &= \E_{\vect{u} \sim P_{\hat{\vect{U}}}}
        \left[Y_{\scm_\graph^\theta(T=d)}(\vect{u}) - Y_{\scm_\graph^\theta(T=t_0)}(\vect{u})\right] \nonumber \\
    &=   \E_{\vect{u} \sim P_{\hat{\vect{U}}}} \left[\int_{t_0}^d \frac{\partial Y_{\scm_\graph^\theta(T=t)}}{\partial t}\, dt\right] \nonumber \\
    &= (d-t_0) \cdot  \E_{\vect{u} \sim P_{\hat{\vect{U}}}}
        \left[\int_{t_0}^d \frac{\partial Y_{\scm_\graph^\theta(T=t)}}{\partial t}\, \frac{1}{d - t_0}\, dt\right] \nonumber \\
    &= (d-t_0) \cdot  \E_{\vect{u} \sim P_{\hat{\vect{U}}}}\left[\E_{t \sim \texttt{Unif}[t_0, d]}
                                    \left[\frac{\partial Y_{\scm_\graph^\theta(T=t)}}{\partial t}\right]\right] \nonumber \\
    &= (d-t_0) \cdot \text{UATD}_{\scm_\graph^\theta}[t_0, d]
    \label{eq:help-ade-uatd}
\end{align}
Therefore,
\begin{align}
    \theta^* = \arg\min_{\theta \in \Theta^n} \text{UATD}_{\scm_\graph^\theta}[t_0, d] =
        \arg\min_{\theta \in \Theta^n} (d-t_0)\cdot \text{UATD}_{\scm_\graph^\theta}[t_0, d] =
        \arg\min_{\theta \in \Theta^n} \text{ATE}_{\scm_\graph^\theta}(d)
\end{align}
For the second part, similar to the proof of Theorem~\ref{theorem:tight-bound}, we have
$\frac{\partial Y_{\scm_\graph^\theta(T=t)}}{\partial t} = g(\theta)$ for some continuous function $g$.
Then,
\begin{align}
    \text{ATE}_{\scm_\graph^\theta}(d) &\stackrel{eq.~\ref{eq:help-ade-uatd}}{=}
        (d-t_0) \cdot \E_{\vect{u} \sim P_{\hat{\vect{U}}}}\left[\E_{t \sim \texttt{Unif}[t_0, d]}
                                    \left[\frac{\partial Y_{\scm_\graph^\theta(T=t)}}{\partial t}\right]\right]
    = (d-t_0) \cdot g(\theta)
\end{align}
The rest of the proof can be directly derived from the proof of
\autoref{theorem:tight-bound}.

\section{ATE and ATD} 
\label{appx:ate-atd}
Here, we provide more intuition on the difference between ATE and ATD. Consider a simple setting with only two random variables $T$ and $Y$, where the causal graph is $T \to Y$. Here, the average treatment effect is identifiable and can be computed (with infinite samples) using the following formula:
\begin{align}
    \text{ATE}_{\scm}(d) = \E_P[Y | T=d] - \E_P[Y | T = t_0]
\end{align}

Now, assume we have a finite dataset $\mathcal{D} = \{(t^{(1)}, y^{(1)}), \cdots, (t^{(n)}, y^{(n)})\}$, where $t^{(i)} \neq t_0$ and $t^{(i)} \neq d$ for all $i \in [n]$. One way to find a (high probability) upper bound on the value of ATE is to solve the following direct optimization:
\begin{align}
    \max_{\theta}\, \text{ATE}_{\scm_\graph^\theta}(d) \quad \text{s.t.}\quad W_1(P_{\scm_\graph^\theta}, P^n) \leq \alpha_n
    \label{eq:ate-direct-opt}
\end{align}
With no assumption on the regularity of the response functions $Y_{\scm_\graph^\theta(T=t)}$, it is possible to find a solution to eq.~\ref{eq:ate-direct-opt} that matches all points in $\mathcal{D}$, i.e., $W_1(P_{\scm_\graph^\theta}, P^n) = 0$, while getting arbitrarily values of $\text{ATE}_{\scm_\graph^\theta}(d)$. See Figure~\ref{fig:ate-intuition} for a demonstration. This shows that ATE, in the continuous treatment setting with finite number of samples, is not well-behaved and it is generally impossible to find informative non-parametric bounds on that (see also~\citet{gunsilius-iv}). On the other hand, ATD, which is the average partial derivative of response function w.r.t. the observed treatment distribution is defined \textit{globally} over the support of $T$. Therefore, it is not possible to maximize ATD arbitrarily without violating the distributional constraint in this setting.

\begin{figure}[t]
\centering
\includegraphics{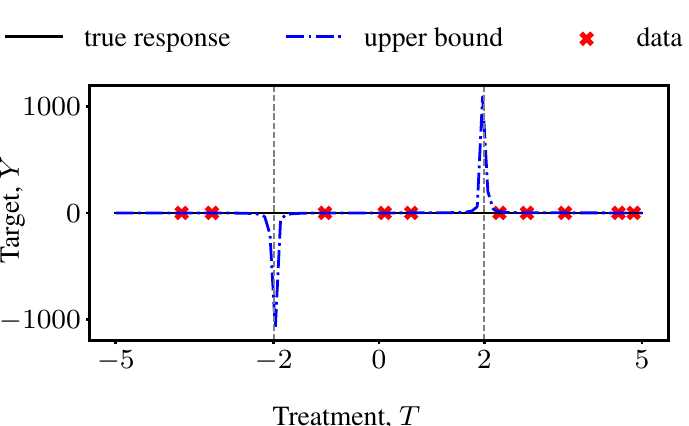}
\caption{(a) Irregularity of the ATE objective. In finite sample setting, it is possible to find a model that matches the empirical distribution, while creating arbitrarily large values of ATE between $T=2$ and $T=-2$.}
\label{fig:ate-intuition}    
\end{figure}

\section{Algorithm}\label{appx:algorithm}
\begin{algorithm}
\caption{Partial Identification of Average Treatment Derivatives}\label{alg:atd}
\textbf{Input}:
Dataset $\mathcal{D} = \{X^{(i)}, T^{(i)}, Y^{(i)}\}_{i=1}^N$,
causal graph $\graph$, Sinkhorn approximation parameter $\epsilon$, learning rate $\gamma$, Lagrange learning rate $\gamma_{L}$\\
\textbf{Output}:
$\hat{\lowerATD}$: lower bounds on ATD \Comment{The algorithm to find $\hat{\upperATD}$ is similar}
\begin{algorithmic}[1]
\Algphase{Phase 1: Matching Distributions}
\State{Initialize $\scm_{\graph}^{\theta^{(0)}}$ and $i=0$}
\While{$S_\epsilon(\hat{D}, D)$ not converged}
    \Comment{Train the model to approximate the observed distribution.}
    \State{$\hat{\mathcal{D}} \leftarrow \{\hat{X}^{(j)}, \hat{T}^{(j)}, \hat{Y}^{(j)}\}_{j=1}^N \sim \scm_{\graph}^{\theta^{(i)}}$} \Comment{Generate dataset from the trained SCM}
    \State{$\theta^{(i+1)} \gets \theta^{(i)} - \gamma \nabla_{\theta} S_\epsilon(\hat{\mathcal{D}}, \mathcal{D})$}
    \State{$i \leftarrow i + 1$}
\EndWhile\\
$\alpha \gets S_\epsilon(\hat{\mathcal{D}}, \mathcal{D})$ \Comment{Set the distributional constraint level to the minimum Sinkhorn loss}
\Algphase{Phase 2: Joint Optimization Phase}
\State{Initialize Lagrange multiplier $\lambda^{(0)}$ and $j=0$}
\While{$\text{ATD}_{\scm_\graph^{\theta^{(i+j)}}}$ not converged}
    \State{$\text{ATD}_{\scm_\graph^{\theta^{(i+j)}}} \gets $ calculate ATD using eq.~\ref{eq:calc-atd}}\\
    { ~ ~ ~ \# Alternate optimization}
    \State{$\bar{\theta}^{(i+j)} \gets \theta^{(i+j)} - \gamma \nabla_{\theta} \text{ATD}_{\scm_\graph^{\theta^{(i+j)}}}$} \Comment{Update the parameters to minimize the ATD}
    \State{$\hat{\mathcal{D}} \leftarrow \{\hat{X}^{(k)}, \hat{T}^{(k)}, \hat{Y}^{(k)}\}_{k=1}^N \sim \scm_{\graph}^{\bar{\theta}^{(i+j)}}$}
    \State{$\theta^{(i+j+1)} \gets \bar{\theta}^{(i+j)} - \gamma \nabla_{\theta} S_\epsilon(\hat{\mathcal{D}}, \mathcal{D})$} \Comment{Update the parameters to minimize the Sinkhorn loss}\\
    { ~ ~ ~ \# Lagrange multiplier update}
    \State{$\lambda^{(j+1)} \leftarrow \lambda^{(j)} + \gamma_{L} (S_\epsilon(\hat{\mathcal{D}}, \mathcal{D}) - \alpha)$}
    \State{$j \leftarrow j + 1$}
\EndWhile
\State{\Return{$\text{ATD}_{\scm_\graph^{\theta^{(i+j)}}}$}}
\end{algorithmic}
\end{algorithm}
\xhdr{Extension of Algorithm~\ref{alg:atd} to ATEs}
We can use a similar method to Algorithm~\ref{alg:atd} for partial identification of average treatment effects (ATEs). The only difference is that, instead of maximizing/minimizing $\text{ATD}_{\scm_\graph^{\theta}}$, we optimize for the value of $\text{UATD}_{\scm_\graph^{\theta}}[t_0, d]$. More concretely, we use the following approximation to estimate $\text{UATD}_{\scm_\graph^{\theta}}[t_0, d]$:
\begin{align}
    \text{UATD}_{\scm_\graph^\theta}[t_0, d] \approx \frac{1}{n} \sum_{i=1}^n \frac{1}{\epsilon} \left[Y_{\scm_\graph^\theta(T=t^{(i)}+\epsilon)}(\vect{u}^{(i)}) - Y_{\scm_\graph^\theta(T=t^{(i)})}(\vect{u}^{(i)})\right]
    \label{eq:calc-ate}
\end{align}
where $\{t^{(i)}\}_{i=1}^n$ are samples from a uniform distribution within $[t_0, d]$ with a Gaussian tail, and $\{\vect{u}^{(i)}\}_{i=1}^n$ are the latent variables generated from a uniform distribution. Note that, the only difference between eq.~\ref{eq:calc-ate} and eq.~\ref{eq:calc-atd} is the distribution of treatment variables, where in the former we use a uniform distribution with Gaussian tail, while the latter uses the same distribution of treatments in the observed dataset. In our experiments, for the uniform distribution with Gaussian tail, we generate samples from $\mathcal{N}(\mu, \sigma)$, where $\mu = \frac{t_0 + d}{2}$ and $\sigma = \frac{d - t_0}{2}$. Then, for each sample within $[t_0, d]$, we generate a new sample from uniform distribution $\texttt{Unif}[t_0, d]$.

\section{Additional Experiments} \label{appx:add-exp}
\subsection{Discrete Setting}
To showcase the generality of our framework, we study two datasets with binary treatments. We consider the binary IV dataset described in~\citet{automated}, where the true value of ATE is not identifiable, but the optimal bound is known. We also use the Front-door binary dataset in~\citet{partial-ident-elias} where the causal effect is identifiable (See~\autoref{appx:DGPs}). Here, the partial derivatives do not exist, so we directly optimize the ATE. Note that, in the discrete setting, the network cannot generate arbitrary large values in the intervention points without violating the distributional constraint. Table~\ref{tab:discrete-results} shows our derived bounds and compares them to the optimal bounds. In the identifiable Front-door causal graph, we find a tight bound over the true ATE. In the non-identifiable IV setting, our bound includes the optimal bound with a small gap.
\begin{table}[t]
    \centering
    \caption{\small The bounds derived by our method over the ATE. The results include the optimal bound.}
    \begin{tabular}{cccc} \toprule[1.5pt]
       Causal Graph  &  Ours &  Optimal Bound & True Value  \\ \midrule[1.5pt]
      Front-door (Discrete)   & (0.4374, 0.5322) & -- & 0.5085 \\
      IV (Discrete) & (-0.5629, -0.0821)& (-0.55, -0.15) & -0.25
    \end{tabular}
    \label{tab:discrete-results}
    \vspace{-1em}
\end{table}

\subsection{ACIC Dataset}
To demonstrate the applicability of our method on datasets with higher-dimensional covariates, we consider the Atlantic Causal Inference Conference (ACIC) 2019 Data Challenge~\citep{acic2019}. The dataset is constructed based on the spam detection data from UCI~\citep{spambase}. The outcome of interest $Y$ is whether an email is marked as spam or not. The treatment variable $T$ is also a binary variable showing if the number of capital letters in an email exceeds a threshold. There are 22 continuous covariates that correspond to certain word frequencies. 

We follow a similar setup as in~\citet{guo2022partial}. In particular, we consider the problem of partial identification under noisy measurements. Here, the data-generating causal graph is the Back-door setting, and the causal effect of $T$ on $Y$ is identifiable. However, we assume measurement error on the covariates by imposing synthetic noise on them and aim to estimate bounds on ATE under this uncertainty. Since our algorithm can incorporate uncertainty through the distributional constraint, it will lead to a valid and informative bound by choosing an appropriate value for $\alpha_n$.

We generate a dataset of $2{,}000$ samples using ACIC's data-generating process.~Similar to~\citet{guo2022partial}, we synthetically add five different levels of Gaussian noise with mean $\in (0.1, 0.2, 0.3, 0.4, 0.5)$ and standard deviation $\in (0.5, 0.5, 1, 1, 1)$. We then run our algorithm on these five noisy datasets and the original noiseless one, $10$ different trials for each. Figure~\ref{fig:acic} illustrates our derived bounds and the true ATE for each of the noise levels. The results always include the actual value of ATE, while not being too conservative. As the noise level increases, our derived bounds get naturally less informative.~\footnote{We heuristically choose the value of hyper-parameter $\alpha_n$ by multiplying the minimum Sinkhorn divergence by factors of $(1.2, 1.3, 1.5, 1.6, 1.7)$ for the noise levels, respectively.}

\begin{figure}[t]
    \centering
    \includegraphics{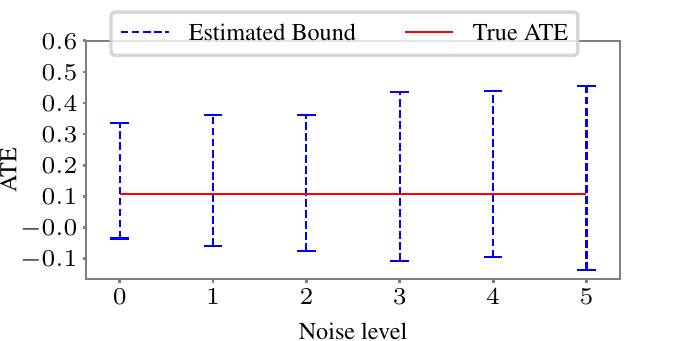}
    \caption{Partial identification of ATE in the ACIC dataset under noisy measurements. Level $0$ corresponds to the original data, and $1$ to $5$ represents the noise levels, from lower to higher measurement error. The result bounds always include the true ATE while not being highly conservative (returning the full support $[-1, 1]$). We take the average over 10 runs with different random seed for each noise level.}
    \label{fig:acic}
\end{figure}

\section{Implementation Choices}\label{appx:implementation}
We use similar neural network architectures for each variable in the causal graph. For hyper-parameter tuning, we search over networks with $\{2, 3, 5\}$ hidden layers, $\{16, 64, 256\}$ neurons in each layer, learning rate within $\{0.001, 0.005, 0.01\}$, and Lagrange multiplier learning rate $\{0.5, 1\}$. The hyper-parameters with the lowest Sinkhorn loss are chosen for each causal graph. We run all experiments for $500$ epochs and use all samples ($N = 5000$) in each iteration. To find the level of distribution constraint $\alpha_n$, we minimize the Sinkhorn loss and check the loss value on a validation set until it does not decrease for $30$ epochs. We then choose the minimum value of Sinkhorn loss on the training data as $\alpha_n$. We use an alternate optimization approach to maximize/minimize the ADE (or ATE). In each step, we first minimize the Sinkhorn loss and then maximize/minimize the value of ADE. We use gradient clipping for the ADE optimizer with the value of $0.2$. The learning rate of the ADE optimizer is half the learning rate for the Sinkhorn loss optimizer. To find the upper (lower) bounds on ATD/ATE, we find their maximum (minimum) value within all steps that satisfy the distributional constraint. All implementation is done in PyTorch Lightning using Adam optimizer. To evaluate and calculate gradients of Sinkhorn loss, we use the "geomloss" library~\citep{sinkhorn}.~\footnote{\url{https://www.kernel-operations.io/geomloss/}}

\section{Data Generating Processes}\label{appx:DGPs}
\xhdr{Linear Back-door}
\begin{align}
    &X \sim \mathcal{N}(2, 1) \nonumber \\
    &T \sim 0.1X^2 - X + \mathcal{N}(1, 2) \nonumber \\
    &Y \sim 0.5T^2 - TX + \mathcal{N}(0, 2)
\end{align}
\xhdr{Nonlinear Back-door}
\begin{align}
    &X_1, X_2, X_3 \sim \mathcal{N}(1, 1) \nonumber \\
    &T \sim X_1 - X_2 + 2X_3 + 2 + \mathcal{N}(0, 3) \nonumber \\
    &Y \sim 3X_1 + X_2 - 0.5X_3 + 3T + \mathcal{N}(0, 2)
\end{align}
\xhdr{Front-door}
\begin{align}
    &U \sim \mathcal{N}(-1, 1) \nonumber \\
    &T \sim U + \mathcal{N}(2, 2) \nonumber \\
    &X \sim 2 T + \mathcal{N}(1, 2) \nonumber\\
    &Y \sim 0.25X^2 -X + U + \mathcal{N}(0, 2)
\end{align}

\xhdr{Linear IV (weak instrument, strong confounding)}
\begin{align}
    &Z_1 \sim \mathcal{N}(-1, 1) \nonumber \\
    &Z_2 \sim \mathcal{N}(0, 1) \nonumber \\
    &U \sim \mathcal{N}(0, 1) \nonumber \\
    &T \sim Z_1 - Z_2 + 0.5U + \mathcal{N}(0, 1) \nonumber \\
    &Y \sim 0.5T - 3U + \mathcal{N}(0, 1)
\end{align}

\xhdr{Nonlinear IV (strong instrument, weak confounding)}
\begin{align}
    &Z_1 \sim \mathcal{N}(-1, 1) \nonumber \\
    &Z_2 \sim \mathcal{N}(0, 1) \nonumber \\
    &U \sim \mathcal{N}(0, 1) \nonumber \\
    &T \sim 3Z_1 +1.5 Z_2 + 0.5U + \mathcal{N}(0, 1) \nonumber \\
    &Y \sim 0.3T^2 - 1.5T + U + \mathcal{N}(0, 1)
\end{align}


\xhdr{Leaky Mediation}
\begin{align}
    &U_1 \sim \mathcal{N}(1, 1) \nonumber \\
    &U_2 \sim \mathcal{N}(-1, 1) \nonumber \\
    &C \sim \mathcal{N}(0, 1) \nonumber \\
    &T \sim C + \mathcal{N}(0, 1) \nonumber \\
    &X_1 \sim T + U_1 + \mathcal{N}(0, 1) \nonumber\\
    &X_2 \sim 2T + U_2 + \mathcal{N}(0, 1) \nonumber\\
    &Y \sim -1.5X_1 + 2X_2 + U_1 + U_2 + C + \mathcal{N}(0, 1)
\end{align}

\xhdr{Binary IV~\citep{automated}}
We use the noncompliance IV dataset in Section~D.1 from~\citet{automated}. The true value of ATE is $-0.25$ while the optimal bound is $[-0.55, -0.15]$.

\xhdr{Binary Front-door~\citep{partial-ident-elias}}
\begin{align}
    &U_1 \sim \texttt{Unif}(0, 1) \nonumber \\
    &U_2 \sim \mathcal{N}(0, 1) \nonumber \\
    &T \sim \texttt{Binomial}(1, U_1)\nonumber \\
    &W \sim \texttt{Binomial}(1, \frac{1}{1 + exp(-T-U_2)}) \nonumber\\
    &Y \sim \texttt{Binomial}(1, \frac{1}{1+exp(W -U_1)})
\end{align}

\xhdr{Linear IV with strong confounding~\citep{padh2022stochastic}}
\begin{align}
    &Z \sim \mathcal{N}(0, 1) \nonumber \\
    &U \sim \mathcal{N}(0, 1) \nonumber \\
    &T \sim 0.5Z + 3U + \mathcal{N}(0, 1) \nonumber \\
    &Y \sim T - 6U + \mathcal{N}(0, 1)
\end{align}

\xhdr{Nonlinear IV with nonlinear interaction between treatment and confounding~\citep{padh2022stochastic}}
\begin{align}
    &Z \sim \mathcal{N}(0, 1) \nonumber \\
    &U \sim \mathcal{N}(0, 1) \nonumber \\
    &T \sim 3Z + 0.5U + \mathcal{N}(0, 1) \nonumber \\
    &Y \sim 0.3T^2 - 1.5TU + \mathcal{N}(0, 1)
\end{align}

 
\end{document}